\pdfoutput=1

\documentclass[11pt]{article}

\usepackage{acl}

\usepackage{times}
\usepackage{latexsym}

\usepackage[T1]{fontenc}

\usepackage[utf8]{inputenc}

\usepackage{microtype}

\usepackage{multirow}

\usepackage{booktabs}

\usepackage{paralist}
\usepackage{mdwlist}
\usepackage{graphicx}
\usepackage{color}

\usepackage[ruled,linesnumbered]{algorithm2e}

\usepackage{xspace}
\usepackage[utf8]{inputenc} 
\usepackage[T1]{fontenc}    
\usepackage{hyperref}       
\usepackage{url}            
\usepackage{booktabs}       
\usepackage{nicefrac}       
\usepackage{microtype}      
\usepackage{graphicx}
\usepackage{multirow,makecell}
\usepackage[american]{babel}
\usepackage{amsmath,amsfonts,amssymb,amsthm,bm}
\usepackage{mathtools}
\usepackage{xspace}

\newtheorem{theorem}{Theorem}

\newtheorem{proposition}[theorem]{Proposition}

\newtheorem{definition}[theorem]{Definition}

\definecolor{maroon}{RGB}{192,80,77}

\def\E{\mathbb{E}}
\def\P{\mathbb{P}}

\def\cA{\mathcal{A}}

\def\cS{\mathcal{S}}

\newcommand{\method}{CRT\xspace}

\title{Provably Confidential Language Modelling}

\author{Xuandong Zhao ~~~ Lei Li  ~~~ Yu-Xiang Wang \\
  University of California, Santa Barbara\\
  \texttt{\{xuandongzhao,leili,yuxiangw\}@cs.ucsb.edu} \\
}

\begin{document}
\maketitle

\begin{abstract}


Large language models are shown to memorize privacy information such as social security numbers in training data. 
Given the sheer scale of the training corpus, it is challenging to screen and filter all privacy data, either manually or automatically.
In this paper, we propose \textbf{C}onfidentially \textbf{R}edacted \textbf{T}raining (\method), a method to train language generation models while protecting the confidential segments.  
We borrow ideas from differential privacy (which solves a related but distinct problem) and show that our method is able to \emph{provably prevent} unintended memorization by randomizing parts of the training process. 
Moreover, we show that redaction with an approximately correct screening policy \emph{amplifies} the confidentiality guarantee. 
We implement the method for both LSTM and GPT language models. 
Our experimental results show that the models trained by \method obtain almost the same perplexity while preserving strong confidentiality\footnote{Our code is available at \url{https://github.com/XuandongZhao/CRT}}.

\end{abstract}

\section{Introduction}
\label{sec:intro}

Language models (LM) have rich real-world applications in, among others, machine translation \cite{Bahdanau2015NeuralMT}, AI chatbots \cite{HosseiniAsl2020ASL}, question answering \cite{Kwiatkowski2019NaturalQA}, and information retrieval \cite{Ganguly2015WordEB}. 
The advent of transformers \cite{Vaswani2017AttentionIA} has fostered a dramatic advancement in the capabilities of generative neural language models, yet they come at a cost to privacy, as the amount of excess parameters in the LM enables it to memorize certain training samples. Recent works show that sensitive user information from the training dataset, such as address and name, can be extracted verbatim from text generation models by querying the LM as an API \cite{Carlini2019TheSS, Carlini2021ExtractingTD, lee2021deduplicating}. How to train a high-performing language model without memorizing sensitive text has become a major research challenge. 

Existing solutions to this problem primarily leverage differential privacy (DP) \citep{dwork2006calibrating}. 
 
Differentially private learning algorithms ensure that an attacker could not infer whether a data point is used for training, let alone extracting the sensitive information within that data point.

However, there are several mismatches between the problem of \emph{privacy} that DP addresses, and our problem of preventing the memorization of sensitive text (henceforth referred to as \emph{confidentiality}). First, confidential information in a natural language dataset is sparse (e.g., the bulk of an email might not carry confidential information). DP's undiscriminating protection for all sentences could be unnecessarily conservative which limits the utility of the trained model. 
Second, what needs to be protected is the content of the sensitive text, rather than the data context. For example, in the sentence \texttt{``My SSN is 123-45-6789.''}, it is the actual SSN that we hope to conceal rather than the general information that someone entered her SSN in a chatbot dialogue. 
Thirdly, the same sensitive content could appear in many data points, which makes the protection of the content more challenging than protecting one data sample. 
These differences motivate us to treat the problem of confidentiality protection in LM separately with new definitions.

Besides DP, we also consider classical techniques of redaction and deduplication. \emph{Redaction} refers to the process of removing sensitive or classified information from a document prior to its publication in governmental and legal contexts.   \emph{Deduplication} is the procedure of detecting and removing identical and nearly identical texts from a corpus. 
The main challenge of applying these techniques is that it is hard to manually redact a gigantic dataset and automated tools are far from being perfect. 


\begin{figure*}[t]
\centering
\includegraphics[width=1.0\linewidth]{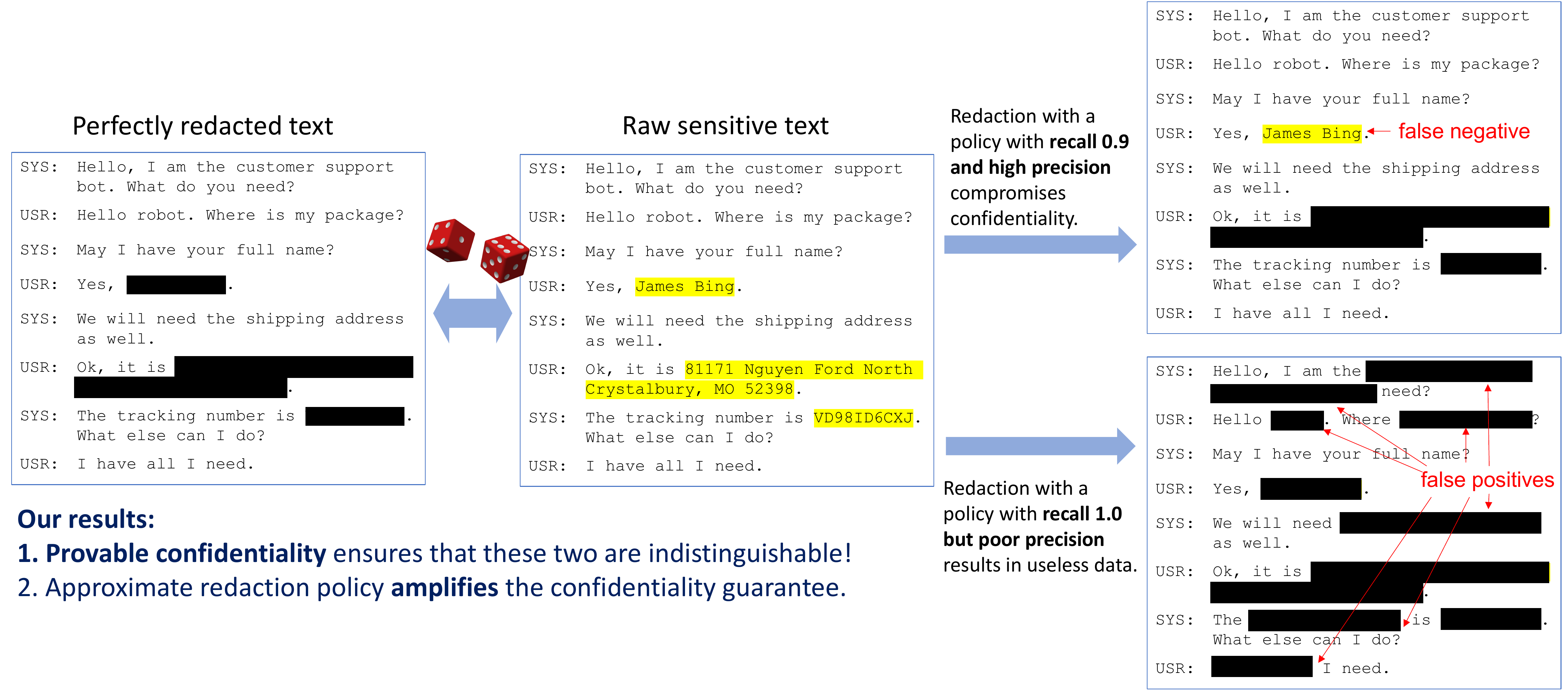}
\caption{An example from simulated dialog dataset \texttt{CustomerSim}. The yellow highlights are confidential content (middle). Left shows the text after \emph{Redaction} by a sequence labeling policy $\pi$. However, if the policy is not perfect, there exists false negative or false positive samples as shown  on the right.}
\label{fig:example}
\end{figure*}

The contribution of this paper is fivefold.
\begin{enumerate}
    \item We show that in the absence of a perfect screening policy, the risk of a language model memorizing sensitive content is real and can be efficiently exploited with only blackbox access to the model even if the learning algorithm satisfies the recently proposed notion of \emph{selective differential privacy} \citep{Shi2021SelectiveDP}.
    \item Inspired by differential privacy, we introduce a new definition of \emph{confidentiality} which precisely quantifies the risk of leaking sensitive text. 
    \item We propose \method to train language generation models while protecting confidential text. The method with deduplication and redaction operations work even under imperfect confidential text labeling policies. 
    \item We theoretically prove that \method, combined with differentially private stochastic gradient descent (DP-SGD), provides strong confidentiality guarantees. 
    \item Our experiments on both MultiWOZ 2.2 and CustomerSim datasets show that different models trained by \method can achieve the same or better perplexity than existing solutions (against the attacks of \citet{Carlini2019TheSS, Carlini2021ExtractingTD}).
\end{enumerate}
To the best of our knowledge, we are the first that rigorously establish the role of deduplication and redaction in achieving provably stronger confidentiality (or the related differential privacy) guarantees; and the first that achieve provably confidentiality in transformer models with only a mild utility loss.

\section{Background \& Related Work}
\label{sec:related}

Next, we briefly introduce the relevant background and discuss the related work to put our work in context.

Language modeling is a fundamental problem in natural language processing  \cite{Devlin2019BERTPO, Howard2018UniversalLM, Raffel2020ExploringTL}. Consider a text sequence that consists of multiple tokens from a vocabulary $\mathcal{V}$, i.e., $\boldsymbol{w}=$ $\left(w_{1}, w_{2}, \ldots, w_{n}\right)$, where $w_{i}$ is the $i$-th token. The goal of language modeling is to construct a generative model of the distribution $\Pr(\boldsymbol{w})$, by applying the chain rule
$
   \Pr(\boldsymbol{w})=\prod_{i=1}^{n} \Pr \left(w_{i} \mid \boldsymbol{w}_{<i}\right).
$
We let $f_\theta(w_i| \boldsymbol{w}_{<i})$ denote the likelihood of token $w_i$ when evaluating the neural network $f$ with parameters $\theta$. A language model is trained to maximize the probability of the data in a training set $\mathcal{W}$, by minimizing the negative log-likelihood over each training example with the loss function
$
    \mathcal{L}(\theta)=-\log \prod_{i=1}^{n} f_{\theta}\left(w_{i} \mid \boldsymbol{w}_{<i}\right).
$
Recurrent neural networks (RNNs) used to be a common choice for the neural network architecture to estimate the probability distribution $\Pr(\boldsymbol{w})$. \cite{Hochreiter1997LongSM, Mikolov2010RecurrentNN}. More recently, large-scale Transformer-based language models have replaced RNNs in state-of-the-art models for all sorts of NLP tasks \cite{Vaswani2017AttentionIA, Radford2019LanguageMA}. Nevertheless, common language models are vulnerable to privacy attacks and possibly expose information about their sensitive training data \cite{Carlini2019TheSS, Carlini2021ExtractingTD}. 

Differentially private (DP) learning methods \citep[see, e.g.,][]{Abadi2016DeepLW} has been applied to language models as a blanket solution for a number of privacy and security risks. \citet{McMahan2018LearningDP} trained an RNN language model with DP guarantees in a federated learning setup. \citet{Anil2021LargeScaleDP} pre-trained BERT under DP on datasets with hundreds of millions of examples. These paper also demonstrated that DP can effectively prevent data-extraction attacks in practice even for algorithms with DP guarantees that are considered too weak from a theoretical-perspective (e.g., $\epsilon = 8$ or $16$). However, the strong protection of DP often results in a substantial drop in the utility of the trained model, which makes them less desirable in practice. In fact, it was recently shown that it is \emph{necessary} for deep learning models to memorize certain training data to achieve high accuracy \cite{feldman2020does}, which suggests that DP or any other techniques that require the model to not memorize any training data will perform poorly in the high-dimensional, power-law distributed real datasets. This motivates us to consider weakened models that only prevent memorizing the sensitive part of the text. 

Recent works \cite{lee2021deduplicating, kandpal2022deduplicating} show that deduplication enables language models to emit memorized text less frequently with same or better accuracy. However, deduplicating training datasets can not prevent all unintended memorization. We combine deduplication and redaction and then apply both techniques to the training process of LM to achieve confidentiality with provable guarantee.

The closest to us is perhaps the work of \citet{Shi2021SelectiveDP}, who proposed \textit{selective differential privacy} (S-DP), which requires indistinguishability between two datasets that differ only on a sensitive message. Correspondingly, they propose an algorithm (Selective DP-SGD) for training RNN that adds noise only to the part of computation that involves sensitive tokens. To define S-DP and to run Selective DP-SGD, one needs to have access to a policy function $F$ which determines which token is sensitive. This requirement limits the applicability of their approach to those applications where such perfect $F$ is known.  We note that even for name-entity recognition the state-of-the-art model is far from being perfect, and which part of the text is considered sensitive is often ambiguous even for human annotators. We will see that naively running Selective DP-SGD with an approximate policy function does not provide a meaningful confidentiality guarantee and is vulnerable to practical data extraction attacks. Finally, we note that in the case when a perfect policy function is available, we can simply use it for redaction, which provides a perfect S-DP with $\epsilon=0$. A big part of our contribution is to refine S-DP to a (slightly different) definition called ``confidentiality'' and to demonstrate that we use an approximate screening policy to amplify the confidentiality parameter.





\section{The \method Method and Theory}
\label{sec:approach}

In this section, we develop our method with provable confidentiality.

\subsection{Formally defining confidentiality}
\label{sec:confidential}
Let the dataset be a collection of $n$ data points --- each being a sequence of tokens. A ``secret'' $x$ is a contiguous subsequence of tokens within a data point that is considered \emph{sensitive} or \emph{confidential}. The goal of our research is to allow us to train language models on such datasets that could contain secrets while provably prevent the model from remembering that these secrets were. We start by defining a formal definition of confidentiality, which uses the following idea of indistinguishability from the DP literature.
\begin{definition}[Indistinguishability]
We say that a pair of distributions $P,Q$ defined on the same probability space are	$(\epsilon,\delta)$-indistinguishable  if for any measurable set $S$, 
\begin{align*}
\Pr_{P}[S] \leq e^\epsilon \Pr_Q[S] + \delta.
\end{align*}
\end{definition}
\begin{definition}[Confidentiality]\label{def:confidentiality}
	We say that $\cA$ ensures that a secret $x$ is $(\epsilon(x),\delta)$-confidential, if for any dataset $D$ that contains $x$ in one of its data points, and an alternative  dataset $D'$ that replaces $x$ in $D$ with a generic \texttt{<MASK>}, it holds that  $(\cA(D),\cA(D'))$  are $(\epsilon(x),\delta)$-indistinguishable.  In addition, we simply say that $\cA$ ensures $(\epsilon,\delta)$-confidentiality  if $\epsilon(x)\leq \epsilon$ for all secret $x$.
\end{definition}
This definition ensures that an attacker cannot distinguish from the output of $\cA$ (the trained language model) whether it was $x$ or \texttt{<MASK>} that was used for training, thus formalizing the idea of confidentiality. The protection should be viewed as relative, rather than absolute. The definition bounds the risk of any bad event  by an multiplicative factor of $e^\epsilon$ and an additive factor of $\delta$, which implies that anything that could happen when we run $\cA$ on the sensitive data could've happened with with similar probability even if $\cA$ runs on an alternative world where these sensitive information are perfectly masked.

\paragraph{Connections to differential privacy.} Our definition of confidentiality is related to (and inspired by) $(\epsilon,\delta)$-differential privacy (DP) but is different in several ways. DP is stronger (and implies confidentiality!) requires $\cA$ to ensure $(\epsilon,\delta)$-indistinguishability for all $D,D'$ that can be modified from each other by adding or removing one individual person / data point (or tokens, depending on the desired granularity); but for $\cA$ to ensure $(\epsilon,\delta)$-confidentiality, it only requires $(\epsilon,\delta)$-indistinguishability for specific $D,D'$ where $D'$ replaces $x$ in $D$ with \texttt{<MASK>}. 
Moreover, it is more informative to define $\epsilon$ as a function of each specific $x$, which is different from DP (it resembles personalized DP \citep{ghosh2015selling}).

The confidentiality definition makes sense for our problem because it protects the content of the sensitive text $x$ rather than its existence. Specifically, a pre-processing algorithm that masks all sensitive text ensures $(0,0)$-confidentiality but does not satisfy any non-trivial DP guarantees. 

Sometimes, it is useful to consider the confidentiality of multiple secret texts. For example, a secret key $x$ could appear multiple times in multiple data points.  Also, there might be multiple secret texts that are correlated to each other such that the knowledge of one would reveal other secrets. 

\begin{definition}[Group Confidentiality]
	We say that $\cA$ ensures that a \textbf{list} of  sensitive texts $\cS := [x_1,...,x_k]$ is  $(\epsilon(\cS),\delta)$-(group) confidential, if for any dataset $D$ that contains $[x_1,...,x_k]$ in up to $k$ data points,  and $D'$ being the version that replaces each element in $\cS$ with \texttt{<MASK>}, it holds that $(\cA(D),\cA(D'))$ are $(\epsilon(\cS),\delta)$-indistinguishable.
\end{definition}
A special case of such group confidentiality is when $\cS$ collects the \emph{all secret text} in $D$, which protects all secret texts uniformly.  We call this \emph{uniform-confidentiality}.  Note that the standard definition of confidentiality also protect every secret $x$, except that it protects each secret $x$ individually, rather than together.

Inspired by the recent development of Bayesian DP \citep{triastcyn2020bayesian}, we also define Bayesian confidentiality as follows. 
\begin{definition}[Bayesian Confidentiality]
	Let $D$ be a dataset that is fixed except a random secret $x \sim\mu$ drawn from some distribution $\mu$. Let $D'$ be obtained by replacing $x$ with \texttt{<MASK>}\footnote{Notice that $D'$ is fixed even though $x$ is random.}. Then $\cA$ ensures $(\epsilon,\delta)$-Bayesian Confidentiality if for any $D'$,
	$(\cA(D),\cA(D'))$ is $(\epsilon,\delta)$-indistinguishable, where $\cA(D)$ is jointly distributed over $x\sim \mu$ and $\cA$.
\end{definition}
The Bayesian confidentiality measures how much information an attacker could gain if he/she's prior knowledge about this secret $x$ is described by the distribution $\mu$.  This is a strict generalization because when $\mu$ is a single point mass at $x$, it recovers Definition~\ref{def:confidentiality}. The additional generality allows us to quantify the stronger confidentiality guarantee against weaker adversaries without complete information.

\subsection{Confidentially redacted training}
In this section we describe the \method method to train language models with provable confidentiality guarantee. 
It includes two pre-processing operations (deduplication and redaction)  and a switching optimization procedure.
The overall idea is to screen the corpus into two separate sets, one public set including sentences with no confidential information, and one private set including sentences containing confidential content. 
We then use normal optimization algorithms (e.g. SGD) on the public set and differential privacy optimizer (e.g. DP-SGD) on the private set.

\paragraph{Deduplication.}
The deduplication procedure $\mathrm{Dedup}$ detects all sentences that appear multiple times in the training data and replace them into a single \texttt{<MASK>} from the second occurrence onwards (\texttt{<MASK>} is for proving purpose). 

\paragraph{Redaction.} The redaction procedure $\mathrm{Redact}_\pi$ takes applies a sequence labelling policy $\pi$ to screen confidential content in the training corpus $D$. $\pi(s, x)=1$ if a token $x$ in a sentence $s$ should be confidential. The labeled span in each detected sentence is replaced with a special token \texttt{<MASK>}. Note that we do not assume the policy is perfect. It may label some non-sensitive tokens as sensitive (false positives) and label some sensitive text as non-sensitive (false negative, or $1-$recall).

$\mathrm{Redact}$ and $\mathrm{Dedup}$ could be implemented manually, but with the large text corpus nowadays it is more common that these procedures are implemented using automated tools.  For example, $\mathrm{Dedup}$ could be implemented efficiently with  just one pass of data using a \emph{bloom filter} \citep{bloom1970space} (or other hashing tricks that also catches near-duplicates). Bloom filter in particular, enjoys the nice property that it could have false positives but never any false negatives.
$\mathrm{Redact}_\pi$ could be realized by a named entity recognition (NER) model or a personal-identifiable information (PII) detector.

\begin{algorithm}[tb]
\SetAlgoLined
\SetKwInOut{Input}{Input}
\Input{Dataset $D$ (after tokenization / splitting), labelling policies $\pi,\pi_c$, number of epochs $T$ }
$D' \leftarrow \textrm{Dedup}(D)$\\
$D'' \leftarrow \textrm{Redact}_\pi(D')$\\
 $D^{pri}\leftarrow \{ s \in D'' | \exists x\in s \text{ s.t. } \pi(s, x)=1 \text{ or } \exists x\subset s \text{ s.t. }\pi_c(s, x) = 1  \}$ \\
 $D^{pub} \leftarrow  \{s\in D'' | s\notin D^{pri}\}$\\ 
 \For{$e = 1,...,T$}
 {
 Run one epoch of SGD with $D^{pub}$\\
 Run one epoch\footnotemark{} of DP-SGD with $D^{pri}$}
\caption{\method }\label{alg:drs-dpdgd}
\end{algorithm}
\footnotetext{ DP-SGD uses Poisson-sampled Gaussian mechanisms (with a random batchsize), thus cannot ensure all data points are seen and some data points might be seen many times. One epoch means the number of iterations that in expectation covers $|D^{pri}|$ data points.}

Finally, \method combines the two pre-processing steps with normal optimizer and DP-SGD, the  standard algorithm for deep learning with differential privacy. A pseudo-code of the algorithm is given in Algorithm~\ref{alg:drs-dpdgd}.

Besides using a sequence labeling policy $\pi$ with balanced precision/recall as part of the redaction process. The algorithm uses another, more conservative, policy $\pi_c$ with nearly perfect recall to decide on the data points that do not contain sensitive text. In the situation when such $\pi_c$ isn't available, we simply choose $\pi_c(s, x) = 1$ for all tokens $x$ in a sentence $s$ and the second part becomes the vanila DP-SGD. It is also important that every data point that contains a \texttt{<MASK>} requires protection.

\subsection{Theoretical analysis}
We analyze the theoretical properties of the above method and show that they result in provable improvements in the (regular, group and Bayesian) confidentiality parameters for any algorithms that are provably $(\epsilon(x),\delta)$-confidential as defined in Section~\ref{sec:confidential}.

The following theorem captures the benefit of redaction in improving confidentiality.
\begin{proposition}[Confidentiality under redaction]\label{prop:masking_conf}
		If $\cA$ ensures $(\epsilon(x),\delta)$-Confidentiality for each token $x$ of sentence $s \in \cS$ ($\cS$ is a corpus), then $\cA\circ \mathrm{Redact}_\pi$ ensures $(\tilde{\epsilon}(x),\delta)$-confidentiality with 
		$$
		\tilde{\epsilon}(x) = \begin{cases}
			\epsilon(x) &\text{ if }\pi(s, x) = 0\\
			0 &\text{ otherwise.}
			\end{cases}
		$$	
		In addition, $\cA\circ \mathrm{Redact}_\pi$ also satisfies $(\tilde{\epsilon}(S),\tilde{\delta}(S))$-group confidentiality with
		\begin{align*}
		    &\tilde{\epsilon}(S) = \sum_{x\in s \& s\in \cS} \epsilon(x) \mathbf{1}(\pi(s, x) = 0),\\ 
		    &\tilde{\delta}(S) = \tilde{k}e^{\tilde{\epsilon}(S)} \delta
		\end{align*}
		
		where $\tilde{k}:=  \sum_{x\in S} \mathbf{1}(\pi(s, x) = 0)$.
\end{proposition}
As an application of the above, if $\cA$ ensures $(\epsilon,\delta)$-confidentiality, and that the empirical recall rates of the redaction policy on $D$ is $1-\gamma$, then the above proposition suggests that $\cA\circ \mathrm{Redact}_\pi$ improves the uniform-confidentiality over applying $\cA$ without redaction by a factor of $\gamma$. The proof is in the appendix.

Redaction also improves Bayesian confidentiality in a way that mirrors the privacy amplification by sampling from the DP literature.
\begin{proposition}[Bayesian Confidentiality under Redaction]\label{prop:masking_bayesian}
	If $\cA$ ensures $(\epsilon,\delta)$-Bayesian Confidentiality with respect to $\mu[x|\pi(s,x)=0]$ for a token $x$ in a  sentence $s$,
	then $\cA\circ \mathrm{Redact}_\pi$ ensures $(\log(1+\gamma (e^{\epsilon}-1)),\gamma\delta)$-Bayesian Confidentiality under $\mu$ if $\pi$ has a  false negative rate (i.e., $1- $``Recall'') of $\gamma$ under $\mu$. 
\end{proposition}
The proposition says that if the redaction policy is accurate for secrets $x\sim \mu$, then we can have a stronger confidentiality parameter that scales roughly at $\tilde{\epsilon} = O(\gamma \epsilon)$.  The idea behind the proof is that over the distribution of $x\sim \mu$, with probability $1-\gamma$, $\mathrm{Redact}_\pi(D) = \mathrm{Redact}_\pi(D')$, thus $\cA\circ \mathrm{Redact}_\pi(D) \equiv \cA\circ \mathrm{Redact}_\pi(D')$.  
With probability $\gamma$, $\mathrm{Redact}_\pi(D), \mathrm{Redact}_\pi(D')$ are different and conditioning on the fact that $\mathrm{Redact}_\pi$ fails to detect $x$. 
Note that $\pi$ is also applied to other text that are not sensitive, and could result in false positives, but they do not matter as the modification of $\mathrm{Redact}_\pi$ to $D$ and $D'$ will be identical. A full proof is given in the appendix.

Next we turn to deduplication.
\begin{proposition}[Group confidentiality under deduplication.]\label{prop:dedup}
	If $\cA$ ensures $(\epsilon(S),\delta(S))$-Group Confidentiality, then 
	$
	\cA\circ \mathrm{Dedup}
	$
	ensures  $(\epsilon(\mathrm{Unique}(S)),\delta(\mathrm{Unique}(S)))$-Group Confidentiality. 
\end{proposition}
Deduplication provides a stronger protection for those cases where some secret  $x$ could appear multiple times in the dataset.



\begin{theorem}\label{thm:guarantee}
Let DP-SGD from Algorithm~\ref{alg:drs-dpdgd} satisfies $(\epsilon,\delta)$-differential privacy. 
\vspace{-0.5em}
\begin{enumerate}
\item Assume $\pi_c(s, x)=1$ for all secret tokens $x$ in a sentence $s$ such that $\pi(s, x)=0$, then
Algorithm~\ref{alg:drs-dpdgd} satisfies $(\epsilon\mathbf{1}(\pi(s, x)=0),\delta)$-confidentiality.
\item Let $S$ be a group containing $m$ unique secrets such that $\pi_c(s, x)=1 \forall x\in s \text{ and } s \in \cS$ and that $\pi$ detects $\tilde{\gamma}$-proportion of the unique secrets in $S$. Then Algorithm~\ref{alg:drs-dpdgd} satisfies  $(\tilde{\gamma}m\epsilon, \tilde{\gamma}m e^{\tilde{\gamma}m\epsilon}\delta)$-group confidentiality for $S$.
\item Let $\pi,\pi_c$ has a a recall of $1-\gamma$ and $1-\delta_2$ respectively on $\mu$, then Algorithm~\ref{alg:drs-dpdgd} satisfies
$(\log(1+\gamma (e^{\epsilon}-1)),\gamma\delta + \delta_2)$-Bayesian Confidentiality for $\mu$.
\end{enumerate}
\end{theorem}

The theorem demonstrates that our  \method algorithm enjoys a full suite of confidentiality guarantees and they all benefit from the deduplication and redaction, particularly if $\pi$ has high recall.

Note that the \method algorithm achieves the worst-case confidentiality guarantee if we have a nontrivial conservative screening policy that outputs $\pi_c(x) = 1$ for \emph{all} secret $x$ that $\pi$ misses, or we simply run vanilla DP-SGD after deduplication and redaction. 
On the other hand, CRT still satisfies Bayesian confidentiality for each $\mu$ depending on the recall rate of $\pi_c$ under $\mu$.

\section{Experiments}
\label{sec:exps}

We evaluate \method by training two types of language model, LSTM and GPT-2, on two datasets: 1) MultiWOZ 2.2, a well-known human-written dialogue dataset and 2) CustomerSim, a simulated dialogue dataset for conversation generation. 

\paragraph{MultiWOZ 2.2} is an already-public dialogue dataset written by crowd-workers, which collects over 10,000 annotated dialogues spanning 8 domains \cite{zang2020multiwoz}. We use this dataset to show how \method works in real-world applications. Following US Department of Labor's guidance\footnote{https://www.dol.gov/general/ppii} on personal-identifiable information (PII), we treat all confidential information (e.g. email address, reference number, telephone number, etc.) as secrets. For the sequence labeling policy $\pi$ and conservative policy $\pi_c$, we build upon an NER model to do redaction. See Appendix \ref{sec:policy_details} for more details.

\paragraph{CustomerSim.}
Following S-DP \citet{Shi2021SelectiveDP}, we simulate a dialog dataset called CustomerSim with synthetic user information. The dialog flow is simulated based on a fixed agenda and the language generation is template-based \cite{Zhao2018ZeroShotDG}. CustomerSim consists of 10 thousand examples and over one million tokens. We treat user name, address, phone number, order, and tracking number as secrets, and use a regular expression tester (regex) to detect them for the redaction process.

\paragraph{Experiment details.}
For LSTM model, we follow the setting in S-DP to choose a one-layer LSTM. Because S-DP requires hidden states of the sensitive input to be protected, it doesn't support more layers nor Bidirectional LSTM. Since the advent of Transformers \cite{Vaswani2017AttentionIA} significantly improves the capabilities of generative language models, we also test transformer-based language model GPT-2 \cite{Radford2019LanguageMA} from HuggingFace \cite{Wolf2019HuggingFacesTS}. As for deduplication, we use \texttt{SHA-1} \cite{jarvinen2004design} hash function to encode sequences to \texttt{SHA-1} hash code and then remove identical sequences based on the same hash code. For Bayesian Confidentiality, we treat the uniform distribution over the secret sequences as the distribution $\mu$. More experiment details can be found in Appendix \ref{sec:exp_details}.
\begin{figure*}[h]
\centering
\includegraphics[width=1\linewidth]{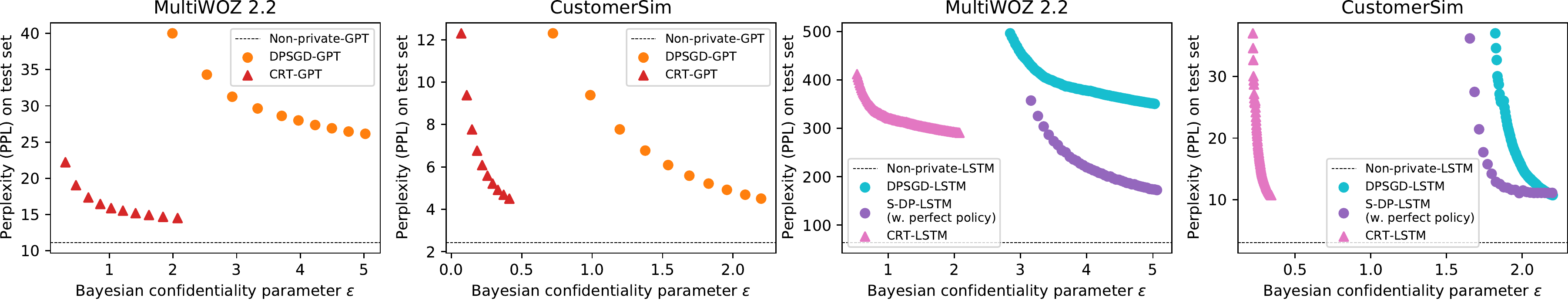}
\caption{Model utility and confidentiality guarantee on MultiWOZ 2.2 and CustomerSim datasets with $\mu$ being a uniform distribution over the secret sequences in each dataset. PPL: Perplexity on the test set. $\epsilon$: Privacy guarantee in Bayesian Confidentiality. We fix $\delta = 8e-5$ for all models. Since Selective DP-SGD with approximate policy gives $\epsilon = +\infty$, we show its result with a perfect screen policy. But when a perfect policy is available, \texttt{Redaction} only gives $\epsilon=0$ and achieves the PPL of vanilla training with no noise added (Non-private-GPT/LSTM). For other models we set $\gamma=0.1$ to show the result under approximate policy.}
\label{fig:main}
\end{figure*}
\paragraph{Baselines.}
For LSTM model, we compare four different training approaches: (1) vanilla SGD (denoted by "Non-private-LSTM"), (2) Selective DPSGD (denoted by "S-DP-LSTM") (3) DPSGD (denoted by "DPSGD-LSTM") and (4) confidentially redacted training (denoted by "\method-LSTM"). While for GPT-2 model, we compare three different training approaches: (1) vanilla SGD (denoted by "Non-private-GPT"), (2) DPSGD (denoted by "DPSGD-GPT") and (3) \method (denoted by "\method-GPT"). Our implementation of S-DP-LSTM model is built upon \citet{Shi2021SelectiveDP}\footnote{https://github.com/wyshi/lm\_privacy}. We run the experiment for the GPT-2 model following \citet{Li2021LargeLM}\footnote{https://github.com/lxuechen/private-transformers}, in which they propose ghost clipping method to alleviate the computational challenge of running DP-SGD with large Transformers. 

All the models are trained five times to reduce randomness, and the parameters are tuned based on the validation set performances.

\section{Experimental Results}

\subsection{Evaluation procedure}
We need to evaluate both model utilities and privacy guarantees of the language models. We measure predictive perplexity (PPL) for the quality of LM. We also analyze the theoretical privacy budget $(\epsilon$, $\delta)$ and test whether language models are private under attacks detailed below. 

\paragraph{Canary insertion attack.}
Canary insertion is proposed as a testing methodology for quantitatively assessing the risk of \emph{unintended memorization} \cite{Carlini2019TheSS}. It inserts random sequences called canaries into the training dataset, then trains the model, and finally calculates the following exposure for the inserted canaries to measure a model’s potential for privacy risks. In our experiment, we randomly generate 10 canaries in the form of "\texttt{My ID is: <random 6-digit number here>}". Each canary is inserted into the training dataset 20 times to generate more salient differences between models.

\begin{definition}[Canary Exposure]
Given a canary $s[r]$, a model with parameters $\theta$, and the randomness space $\mathcal{R}$, the exposure of $s[r]$ is
$$
\operatorname{exposure}_{\theta}=\log _{2}|\mathcal{R}|-\log _{2} \operatorname{rank}_{\theta}(s[r])
$$
\end{definition}
After training, we calculate empirical model perplexity for all possibly-instantiated canaries and list them in sorted order. Then we can get the canary exposure based on the rank of a specific canary sequence $\operatorname{rank}_{\theta}(s[r])$ and the number of all possible candidates $|\mathcal{R}|$. In our setting, we show the highest canary exposure in 10 canaries. For example, if a canary ranks 1st among 1M candidates, the canary exposure is 19.93.

\paragraph{Membership inference attack.}
Membership Inference is a widely used privacy attack method. Given a non-privately trained model, an adversary can predict whether or not a particular example was used to train the model. We adopt the membership inference attack in \citet{Carlini2021ExtractingTD}. The general idea is to calculate the given samples’ perplexities under the model, rank them and choose the ones with the lowest perplexities, i.e., highest likelihood by the model. We can think of this process as training a binary classifier based on the perplexity feature. We also implement the group membership inference attack to show the group confidentiality. More details about the implementation can be found in the Appendix \ref{sec:mi_details}.

\subsection{Overall performance}

Figure \ref{fig:main} presents the results of model utilities and confidentiality guarantees across our models of interest on MultiWOZ 2.2 and CustomerSim datasets. Each point denotes a model for different epochs in a training process. Since the X-axis is $\epsilon$ in Bayesian Confidentiality (the lower the better) and the Y-axis is perplexity (the lower the better), a perfect model will lie in the bottom-left corner.
\method-GPT and DPSGD-GPT in general, perform better than S-DP-LSTM, \method-LSTM and, DPSGD-LSTM on the test sets. Our model \method-GPT's performance is close to Non-private-GPT in terms of PPL while preserving strong confidentiality. Besides, \method-GPT is better than DPSGD-GPT manifested by a much lower $\epsilon$, which demonstrates that approximately correct screening policy amplifies the confidentiality guarantee. 

Differences can be witnessed in the results from two different datasets: the models trained on CustomerSim achieve overall better performances than those trained on MultiWOZ. We think it's due to the fact that CustomerSim contains simple dialogs from template-based simulations. 

\subsection{Attack results}
Figure \ref{fig:result1}, \ref{fig:result2}, and \ref{fig:result3} present the results from canary insertion attack and individual/group membership inference attack on MultiWOZ 2.2 and CustomerSim datasets. The X-axis is the false negative rate $\gamma$ of screening policy $\pi$, ranging from 0.0 to 0.5; the Y-axis is the canary exposure (in Figure \ref{fig:result1}) and membership inference accuracy (in Figure \ref{fig:result2} and \ref{fig:result3}), which measures the effectiveness of the attacks. The lower the canary exposure or inference accuracy, the better protection the model provides against the attacks.

\begin{figure}[ht]
\centering
\includegraphics[width=1.0\linewidth]{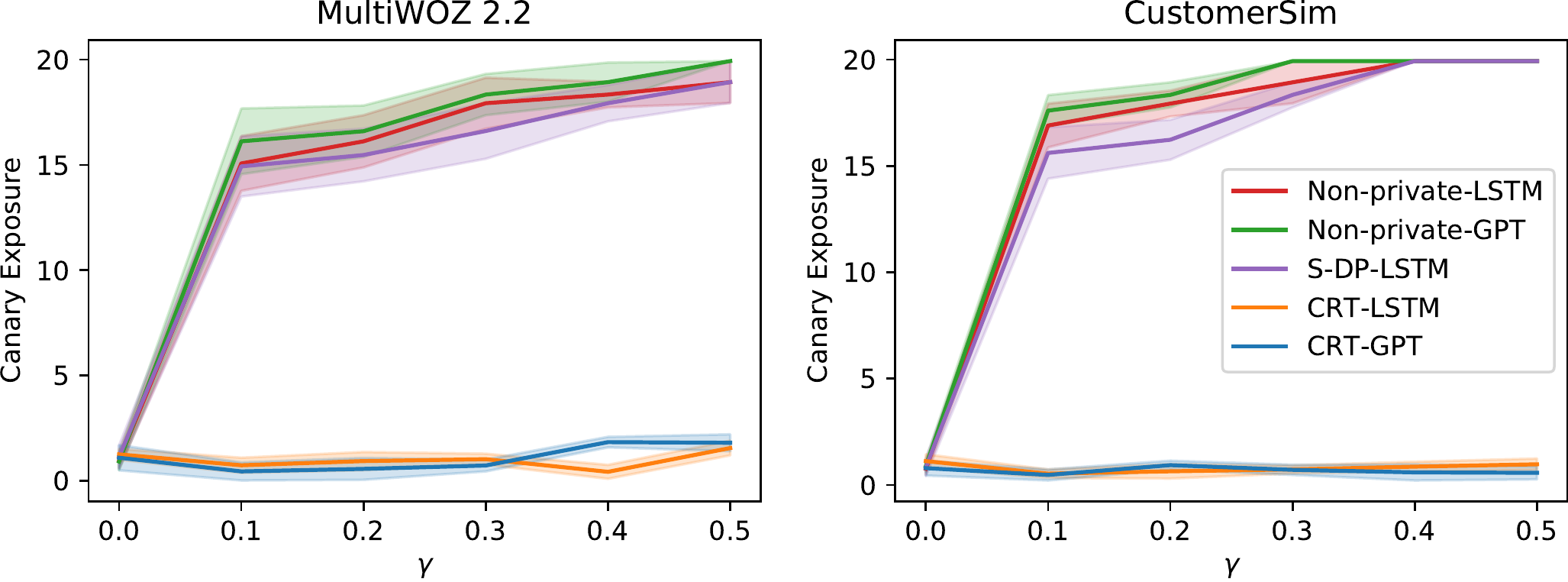}
\caption{Canary insertion attack result. \method achieves almost 0 canary exposure, which means it can prevent unintended memorization.}
\label{fig:result1}
\end{figure}

For canary insertion attack, it can be seen from Figure \ref{fig:result1} that the canary exposures for \method-LSTM and \method-GPT are both close to 0 which thus guarantee excellent confidentiality. Non-private-LSTM and Non-private-GPT with mask can also attain great protection at perfect screening policy accuracy ($\gamma=$ 0), nonetheless a rise in $\gamma$ results in a sharp increase in the exposure. It should be noticed that S-DP-LSTM also has high exposure, similar to Non-private models, given any $\gamma$ above 0. This is because that many sensitive data has been falsely identified as non-sensitive by the approximate policy, S-DPSGD does not protect these false negative samples and hence a privacy leakage. 

\begin{figure}[htbp]
\centering
\includegraphics[width=1.0\linewidth]{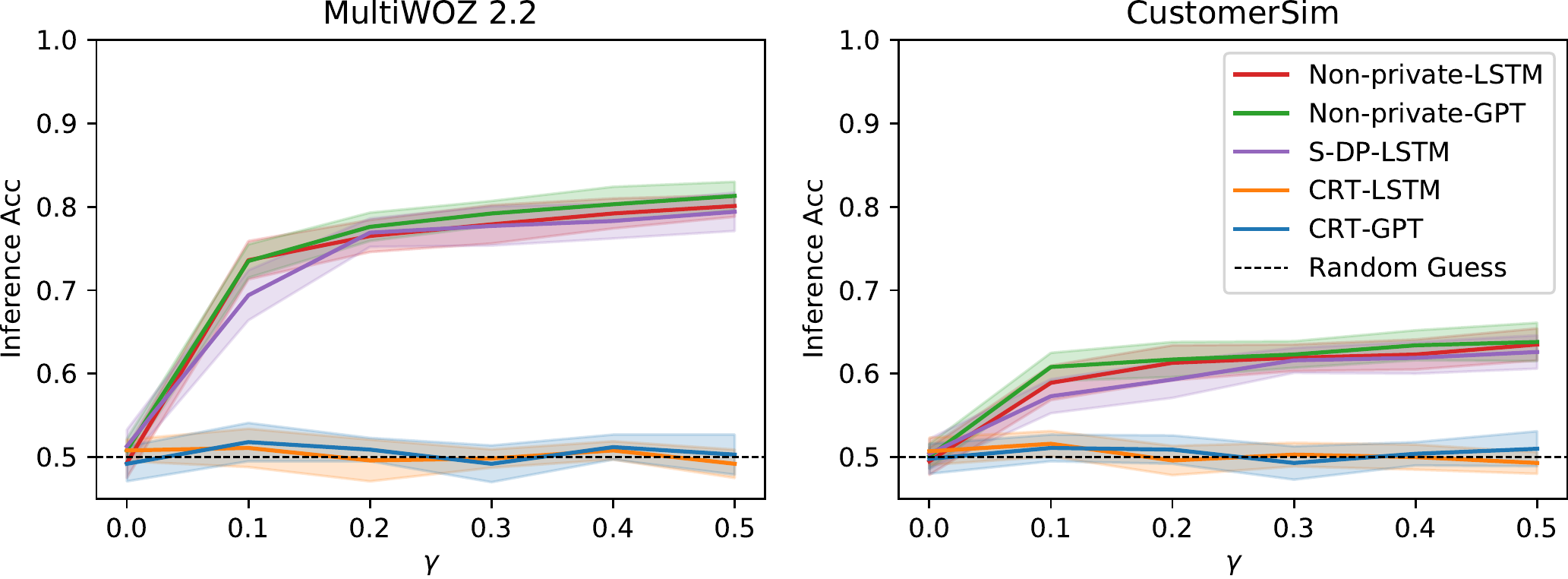}
\caption{Membership inference attack result. \method attains nearly 50\% accuracy, indicating that the adversary could not infer whether a data point is used for training. }
\label{fig:result2}
\end{figure}

\begin{figure}[htbp]
\centering
\includegraphics[width=1.0\linewidth]{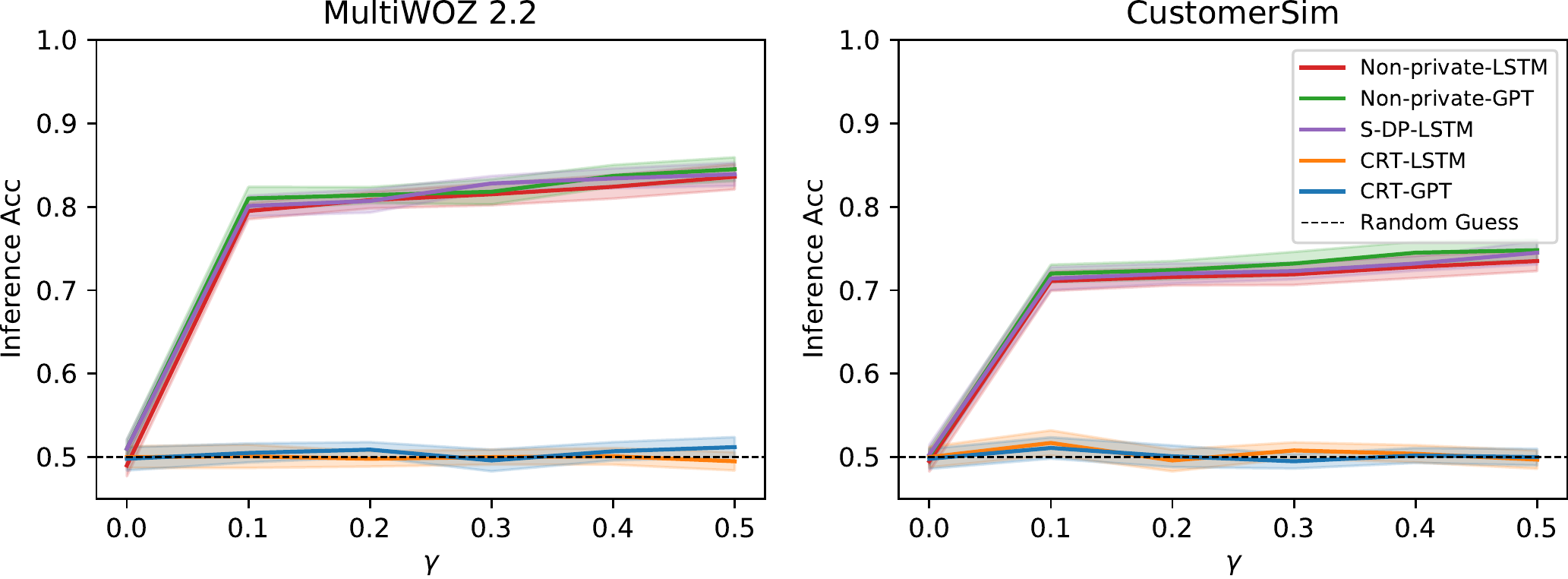}
\caption{Group membership inference attack result.}
\label{fig:result3}
\end{figure}

For membership inference attack, we compare the inference accuracy with the benchmark value of 0.5, which equals the random guess performance. In Figure \ref{fig:result2} and \ref{fig:result3}, we see that \method-LSTM and \method-GPT align well with the 0.5 horizontal line, suggesting that they are rather safe to the attack. The inference accuracy for Non-private-LSTM/Non-private-GPT/S-DP-LSTM, in contrast, surges above 0.5 as the false negative rate $\gamma$ deviates from 0.0, indicating that these models become vulnerable to the attack under non-perfect screen policy. In addition, Non-private and S-DP models show even worse protection under the group attack than the individual one in view of a higher inference accuracy at certain $\gamma$. 
\begin{figure}[htbp]
\centering
\includegraphics[width=0.9\linewidth]{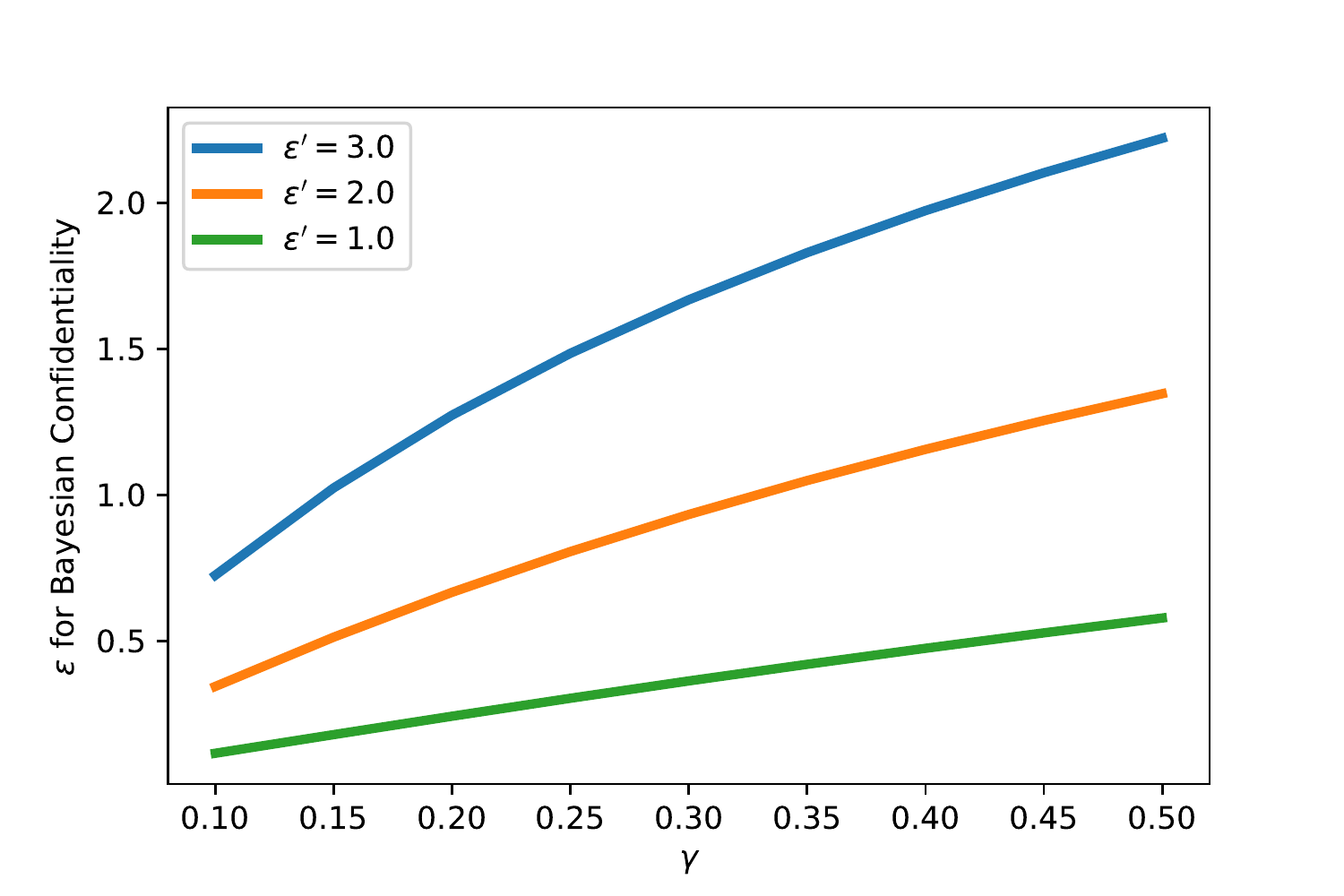}
\caption{Bayesian Confidentiality amplification result. \method helps to amplify the confidentiality guarantee.}
\label{fig:amp}
\end{figure}
\subsection{\method amplifies Bayesian Confidentiality guarantees}

Figure \ref{fig:amp} shows that confidentially redacted training can help to amplify the confidentiality guarantees. We set the $\epsilon'$ in DP-SGD fixed and show the corresponding $\epsilon$ in Bayesian Confidentiality with different screen policy $\pi$. Both $\epsilon'$ and $\epsilon$ are for $\delta=8e-5$. If the approximately screening policy $\pi$ has a high recall ($\gamma$ is small), we will achieve much improvement in the Bayesian Confidentiality parameter $\epsilon$ by deduplication and redaction. For example, with $(\epsilon'=1.0, \gamma=0.1)$, we reduce the $\epsilon$ to 0.12.

\section{Conclusion}
\label{sec:conclusion}

In this paper, we propose confidentially redacted training (\method), a method to train language models while protecting the secret texts. We introduce a new definition of confidentiality which quantifies the risk of leaking sensitive content. We prove the effectiveness of \method both theoretically and empirically on multiple datasets and language models.

\section{Broader Impact}
\label{sec:impact}
This work will alleviate ethical concerns of large-scale pre-trained language models. 
This paper provides one promising solution to an important aspect of NLP: training high quality language models for text generation without compromising confidential information. 
The current use cases of language models involve pretraining on public web corpus and fine-tuning on individual application data. 
However, the private application specific data often contains user-generated sensitive information. 
The proposed method in this paper aims to use as much individual fine-tuning data as possible, while does not leak or memorize any confidential information with provable guarantees. 
Without the method, one has to either use the general pretraining LM without fine-tuning or manually filter sensitive information and fine-tuning on the remaining. 
It can be applied in broader applications that need language models or text generation models. 

In our experiments, we use a simulation scheme to mimic confidential content in a real corpus. We did not compromise any real user's confidential information.

\section*{Acknowledgements}
The work was partially supported by NSF Award \# 2048091. XZ was supported by UCSB Chancellor’s Fellowship. We would like to thank the anonymous reviewers for their thoughtful comments. We would also like to thank Siqi Ouyang for the helpful discussion and Yang Gao for polishing up the draft.

\bibliography{paper}

\appendix
\onecolumn

\section{Appendix}
\label{sec:appendix}
\subsection{Illustration of our proposed algorithm}

\begin{figure}[h!]
\centering
\includegraphics[width=0.9\linewidth]{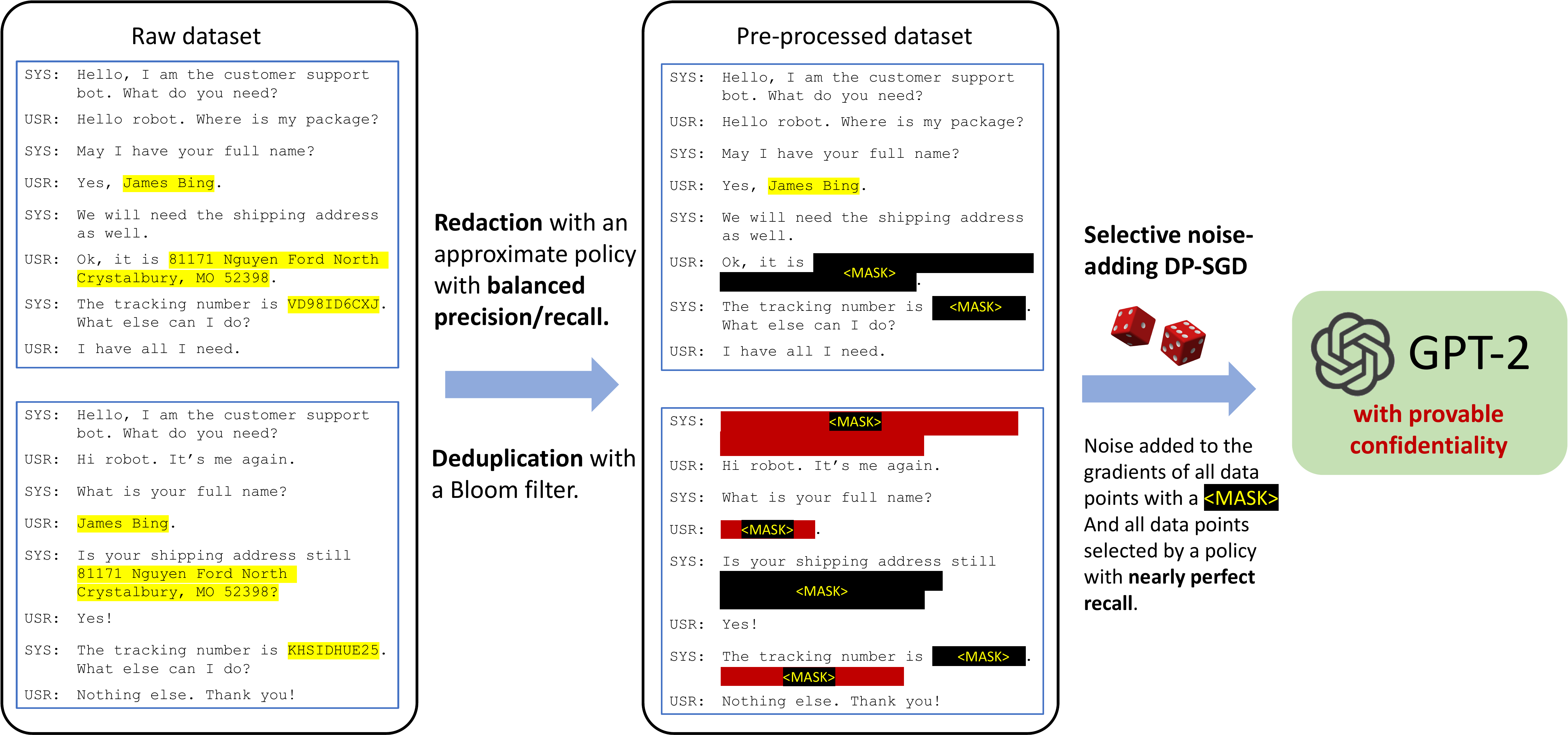}
\caption{An illustration of our proposed algorithm on a dataset with two data points. The first data point is the example from Figure~\ref{fig:example}, and the second data point is modified to illustrate the various aspects of the pre-processing steps.  The red-colored mask indicates the masks produced by deduplication just for illustration purposes. In the algorithm, both of them replace a sequence of tokens with the same special token \texttt{<MASK>}.}
\label{fig:illus_drsdpsgd}
\vspace{-2mm}
\end{figure}

\subsection{Proofs of technical results}
\begin{proof}[Proof of Proposition~\ref{prop:masking_conf}]
	The first statement straigtforwardly follows from that $ \mathrm{Redact}_\pi(D) =  \mathrm{Redact}_\pi(D')$ if $\pi(s, x) = 1$ and that $\mathrm{Redact}_\pi(D) $ and $\mathrm{Redact}_\pi(D')$ remain a pair of neighbors differing by only $x$.  The group confidentiality claims follows from the standard calculation of small group privacy from differential privacy, which applies the (single $x$) confidentiality iteratively. Let $\tilde{D} = \mathrm{Redact}_\pi(D)$, $\tilde{D}' = \mathrm{Redact}_\pi(D')$ and $\tilde{S}=[x_1,...,x_{\tilde{k}}]$ be the list of $S$ that are not masked by $\pi$.  For any measurable event $E$
	\begin{align*}
	\P[\cA\circ\mathrm{Redact}_\pi(D)\in E]  &= \P[\cA(\tilde{D})] \leq  e^{\epsilon{x_1}}\P[\cA(\tilde{D}_{-x_1, + \texttt{<MASK>}})\in E]  + \delta \\
	\leq&  e^{\epsilon{x_1} + \epsilon(x_2)}\P[\cA(\tilde{D}_{-x_{1,2}, + \texttt{<MASK>}^2})\in E] +  e^{\epsilon{x_1}}\delta + \delta\\
	&...\\
	\leq&  e^{\sum_{i=1}^{\tilde{k}}\epsilon{x_i}}\P[\cA(\tilde{D}')\in E] + \delta (1 + e^{\epsilon{x_1}} + e^{\epsilon{x_1}+\epsilon{x_2}} + ... + e^{\epsilon{x_1}+...+\epsilon{x_{\tilde{k}-1}}})\\
	\leq& e^{\tilde{\epsilon}(S) }\P[\cA\circ\mathrm{Redact}_\pi(D')\in E]   + k e^{\tilde{\epsilon}(S) } \delta  
	\end{align*}
\end{proof}

\begin{proof}[Proof of Proposition~\ref{prop:masking_bayesian}]
	Consider a dataset $D$ (in which one of the data point has $x\sim \mu$) and a fixed $D'$. 
	Denote the probability distributions $p,q,r$ as shorthands for
	\begin{align*}
		p \sim \cA\circ \mathrm{Redact}_\pi(D) | \pi(s, x)=1\\
		q\sim\cA\circ \mathrm{Redact}_\pi(D) | \pi(s, x)=0\\
		r\sim \cA\circ \mathrm{Redact}_\pi(D') | \pi(s, x)=0
	\end{align*}
	Moreover, we use $\alpha p + (1-\alpha) q$ to denote the mixture distribution that samples from $p$ with probability $\alpha$ and $q$ with probability $1-\alpha$.
	
	Recall that the Hockey-Stick-divergence characterization of $(\epsilon,\delta)$-indistinguishsability \citep{barthe2013beyond}, which says that $(P,Q)$ are   $(\epsilon,\delta)$-indistinguishsable if and only if
	$$
	H_{e^\epsilon}(P\|Q) :=  \E_{y\sim Q}[ (\frac{dP}{dQ}(y) - e^\epsilon)_+  ] \leq \delta.
	$$
	It suffices for us to bound the following quantity:
	\begin{align*}
		&H_{1+\gamma(e^\epsilon-1)}(\cA\circ \mathrm{Redact}_\pi(D) \| \cA\circ \mathrm{Redact}_\pi(D')) = H_{e^\epsilon}((1-\gamma)p  +\gamma q\| (1-\gamma)p  +\gamma r)\\
		=&\gamma H_{e^\epsilon}( q\| (1-\beta)p  +\beta r) \leq \gamma \left( (1-\beta)H_{e^\epsilon}( q\| p)  +\beta H_{e^\epsilon}( q\| r)\right)
	\end{align*}
	where $\beta = \frac{1+\gamma(e^\epsilon-1)}{e^\epsilon}$. In the above, the second line follows from Theorem 2 of \citep{balle2018privacy} (an identity called ``Advanced Joint Convexity'' by the authors) and the inequality is due to the (standard) joint convexity of the Hockey-Stick divergence. It remains to bound $H_{e^\epsilon}( q\| p)$ and $H_{e^\epsilon}( q\| r)$.
	
	Check that $p,r,\cA\circ \mathrm{Redact}_\pi(D')$ are identically distributed and that $H_{e^\epsilon}( q\| r) \leq \delta$ by our assumption on $\cA$'s Bayesian confidentiality guarantee w.r.t. $\mu(x | \pi(s, x)=0)$. This completes the proof.
\end{proof}

\begin{proof}[Proof of Proposition~\ref{prop:dedup}]
	The proof is straightforward as $\mathrm{Dedup}(D)$ differs from $\mathrm{Dedup}(D')$  only by $\mathrm{Unique}(S)$.
\end{proof}

\begin{proof}[Proof of Theorem~\ref{thm:guarantee}]
The proof for the first statement follows from the fact that DP implies $(\epsilon,\delta)$-confidentiality and Proposition~\ref{prop:masking_conf}. Notably, if $\pi_c$ catches all $x$ that is missed by $\pi$, then we get that for all secret $x$, $\epsilon(x) \leq \epsilon$.

The proof of the second statement applies Proposition~\ref{prop:dedup} and the second part of Proposition~\ref{prop:masking_conf}.

The proof of the third statement applies Proposition~\ref{prop:masking_bayesian} but requires a separate treatment of the case when $x$ is missed by both $\pi$ and $\pi_c$.
Let the event that a secret $x$ is not selected by the conservative policy be $E$ and let $\cA$ be a generic algorithm satisfying $(\epsilon,\delta_1)$ Bayesian confidentiality under $\mu$,
\begin{align*}
\P[\cA(D)\in S] &\leq    \P[\cA\circ\mathrm{Redact}_\pi(D)\in S\subset E^c] + \delta\\
&\leq  e^{\epsilon} \P[\cA(D')\in S\subset E^c] + \delta_1 + \delta_2 \\
&\leq e^{\epsilon} \P[\cA(D')\in S]  + \delta_1 + \delta_2.
\end{align*}
This completes the proof.
\end{proof}

\subsection{More details on experiments}\label{sec:exp_details}
We choose the one-layer LSTM with an embedding size of 200 and a hidden size of 200.  We choose distill-gpt2\footnote{https://huggingface.co/distilgpt2} as the GPT-2 model, which has 6 layers, 768 dimension and 12 heads. Vocabulary size for GPT-2 is 50257.
Our experiments are conducted on NVIDIA TITAN-Xp GPU. For LSTM models, we tune the hyperparameters of the learning rate (lr) among \{20, 10, 5, 1, 0.1, 0.05, 0.01\}, batch size (bs) and the epochs among \{5, 10, 30, 50, 100\}. We finally choose \{lr=20, bs=256, epochs=50\} for Non-private-LSTM, \{lr=0.1, bs=5, epochs=50\} for S-DPSGD-LSTM and \{lr=0.05, bs=10, epochs=100\} for \method-LSTM. The same set of hyperparameters are tuned for GPT model as well. Our final choice for DPSGD-GPT/\method-GPT model is \{lr=5e-4, bs=256, epochs=10\}. The actual run-time of algorithms depends on implementation details. Here, we outline estimates of the run-time for training. Running one epoch on \method-LSTM takes 2 hours wheras the same task on \method-GPT only takes 30 minutes since the implementation of \citet{Li2021LargeLM} is highly efficient. We use autodp\footnote{https://github.com/yuxiangw/autodp}, an automating differential privacy computation for the privacy analysis. Noise scale $\sigma$ is calculated numerically so that a DP budget of $(\epsilon, \delta)$ is spent after $T$ epochs.

\subsection{Redaction policy details}\label{sec:policy_details}
We build the sequence labeling policy based on trimming one NER model\footnote{https://huggingface.co/flair/ner-english-ontonotes-fast} trained on OntoNotes-5.0 \cite{ontonote} dataset. We modify the last layer of the NER model and set the threshold for the output scores to enable abnormal/sensitive data detection. For the screen policy $\pi$, we set the threshold to be 0.3 for all predictions with OntoNotes tags. For the conservative policy $\pi_c$, we select all predictions with tags and all plain texts with scores smaller than 0.9 to be sensitive data. We manually label 200 data points and find that the conservative policy $\pi_c$ can achieve 100\% recall with lots of false positives and that $\pi$ can achieve 90\% recall with few false positives.  

\subsection{Membership inference attack details} \label{sec:mi_details}
In our experiments, we manually construct a dataset with 2000 sequences. We select 1000 sequences from the protected secrets used in the training data. And we randomly generate 1000 samples of similar format which are not used in the training data. In this way, a random guess generates an accuracy of 50\%. For MultiWoz 2.2, we use sentences with reference numbers as the secrets. For CustomerSim, we choose customer addresses as the secrets. 

In order to show group confidentiality guarantees, we also conduct group membership inference attack. In this setting, we construct a dataset with 2000 groups, each of which includes 20 sentences. One half of the groups are ``sensitive groups" with all 20 sentences drawn from protected secrets and the other half are "insensitive groups" with all 20 sentences being random. We build the classifier based on the sum of the perplexities in one group.

\subsection{``The devil is in the details'' -- how things could go wrong with seemingly inocuous changes to the algorithm.}\label{sec:devils}

In this section, we highlight various aspects of our algorithms and why certain choices in the pre-processing steps need to be done in the specific way we recommend for our results to hold for them.

\begin{enumerate}
\item It is important that the definition of confidentiality is defined with respect to a perfectly redacted version of the dataset. If we define it as in selective differential privacy, then there will \emph{not} be an amplification effect from redaction. This is because if we replace a secret $x$ that can be detected by $\pi$ with another $x'$ that cannot be detected by $\pi$, then even if $x$ is replaced with \texttt{<MASK>}, $x'$ will not be and the two datasets are still different after redaction.  In addition, the S-DP definition will not be useful for us we do not know how to define a confidentiality parameter specific for each $x$ or Bayesian confidentiality parameter for each $\mu$
    \item Tokenization and  splitting into individual ``sentences'' (data points) should go before redaction / de-duplication. Otherwise redaction with an approximate screening policy and with an ideal screening policy, or deduplication may cause \emph{misalignments}, resulting in almost all data points being different in the preprocessed version of $D$ and $D'$.
    \item Each data point should contain only ``whole'' natural sentences, otherwise the sensitive part of a natural sentence could split into two data points.
    \item Deduplication steps should replace duplicate text with the same \texttt{<MASK>}, otherwise \texttt{<MASK\_Dedup>} and \texttt{<MASK\_Redact>} are not the same so even if all secrets are masked,  there will be a difference between the pre-processed versions of $D$ and its neighbor, while in our approach there are no differences and we achieve perfect confidentility (with $\epsilon=0$).
    \item Any data point containing \texttt{<MASK>} needs to be put in $D^{pri}$.  This is because otherwise our algorithm that works on $D'$ will be a deterministic algorithm that is perfectly distinguishable from the alternative world where the algorithm is random because the approximate policy $\pi$ fails to redact certain secrets $x$.
    \item In the DP-SGD algorithm, the sampled minibatches should contain the whole minibatch from $D^{pri}$ or the whole minibatch from $D^{pub}$. Otherwise the noise always need to be added and the algorithm is identical to the vanilla DP-SGD, and there is no benefit of having a portion of the data being public comparing to all of the data are private.
\end{enumerate}
  


\end{document}